\documentclass[conference]{IEEEtran}

\IEEEoverridecommandlockouts
\usepackage{booktabs}    
\usepackage{makecell}      
\usepackage{array}              
\usepackage{tabularx}       
\usepackage[table,HTML]{xcolor} 
\usepackage{graphicx}          
\usepackage{caption}            
\usepackage{amsmath,amssymb,amsfonts}
\newtheorem{proof}{Proof}

\newtheorem{theorem}{Theorem}
\usepackage{algorithmic}     
\usepackage{multirow} 
\usepackage[ruled,vlined,lined,ruled,linesnumbered]{algorithm2e}
\usepackage{textcomp}          
\usepackage{cite}         

\def\BibTeX{{\rm B\kern-.05em{\sc i\kern-.025em b}\kern-.08em
    T\kern-.1667em\lower.7ex\hbox{E}\kern-.125emX}}
\begin{document}

\title{TriCon-SF: A Triple-Shuffle and Contribution-Aware Serial Federated Learning Framework for Heterogeneous Healthcare Data\\
}

\author{\IEEEauthorblockN{1\textsuperscript{st} Yuping Yan}
\IEEEauthorblockA{\textit{School of Engineering} \\
\textit{Westlake University}\\
Hangzhou 310030, China \\
yanyuping@westlake.edu.cn}
\and
\IEEEauthorblockN{2\textsuperscript{nd} Yizhi Wang}
\IEEEauthorblockA{\textit{School of Engineering} \\
\textit{Westlake University}\\
Hangzhou 310030, China \\
wangyizhi@westlake.edu.cn}
\and
\IEEEauthorblockN{3\textsuperscript{rd} Yuanshuai Li}
\IEEEauthorblockA{\textit{School of Engineering} \\
\textit{Westlake University}\\
Hangzhou 310030, China \\
}
\and
\IEEEauthorblockN{4\textsuperscript{th} Yaochu Jin}
\IEEEauthorblockA{\textit{School of Engineering} \\
\textit{Westlake University}\\
Hangzhou 310030, China \\
jinyaochu@westlake.edu.cn}
}

\maketitle

\begin{abstract}
Serial pipeline training is an efficient paradigm for handling data heterogeneity in cross-silo federated learning with low communication overhead. However, even without centralized aggregation, direct transfer of models between clients can violate privacy regulations and remain susceptible to gradient leakage and linkage attacks. Additionally, ensuring resilience against semi-honest or malicious clients who may manipulate or misuse received models remains a grand challenge, particularly in privacy-sensitive domains such as healthcare. To address these challenges, we propose TriCon-SF, a novel serial federated learning framework that integrates triple shuffling and contribution awareness. TriCon-SF introduces three levels of randomization by shuffling model layers, data segments, and training sequences to break deterministic learning patterns and disrupt potential attack vectors, thereby enhancing privacy and robustness. In parallel, it leverages Shapley value methods to dynamically evaluate client contributions during training, enabling the detection of dishonest behavior and enhancing system accountability. Extensive experiments on non-IID healthcare datasets demonstrate that TriCon-SF outperforms standard serial and parallel federated learning in both accuracy and communication efficiency. Security analysis further supports its resilience against client-side privacy attacks.

\end{abstract}


\section{INTRODUCTION}
Healthcare data, such as medical images, electronic health records (EHRs), and clinical notes, play a vital role in supporting clinical workflows such as diagnostic assistance \cite{zhou2024ai}, risk prediction, and postoperative assessment \cite{arora2023value}. However, the highly sensitive nature of healthcare data and strict privacy regulations (e.g., HIPAA \cite{annas2003hipaa}, GDPR \cite{protection2018general}) severely limit the feasibility of centralized data collection and sharing. As a result, medical and industrial entities often face the challenge of isolated and fragmented datasets. Additionally, healthcare data is inherently heterogeneous, varying significantly in modality, distribution, and patient demographics, which poses further challenges to collaborative and scalable learning \cite{hu2025split}.

Federated Learning (FL) \cite{kairouz2021advances} has emerged as a promising paradigm to address these challenges by enabling distributed model training across decentralized clients without requiring raw data exchange. Parallel FL methods such as FedAvg \cite{mcmahan2017communication} have demonstrated significant success by aggregating model updates from participating clients via a central server, thereby improving model generalization capability while preserving data locality. However, parallel FL may not be optimal for all scenarios, particularly in cross-silo healthcare settings where communication efficiency, fast convergence, and resilience to non-independent and identically distributed (non-IID) data are essential \cite{chang2018distributed,zhu2021federated}.

To overcome the limitations of parallel FL, serial approaches such as Cyclic Weight Transfer (CWT) \cite{chang2018distributed} and Split Learning (SL) \cite{vepakomma2018split} have been proposed. These methods eliminate the need for centralized aggregation by enabling sequential, client-to-client model training. Serial FL is particularly well suited for healthcare scenarios, where reducing communication overhead and avoiding central coordination are crucial \cite{qu2022rethinking,sheng2023federated}. Despite these advantages, two fundamental challenges remain unaddressed, as shown in Figure \ref{fig:problem}:
\begin{itemize}
    \item \textbf{Privacy risk in model transfer:} The direct handoff of models between clients may violate privacy regulations and remain vulnerable to gradient leakage and linkage attacks. In addition, semi-honest clients may tamper with or exploit received models without proper accountability.
    \item \textbf{Sensitivity to data heterogeneity:} Serial FL is particularly sensitive to non-IID data, often suffering from catastrophic forgetting and unstable convergence due to sequential training on disjoint data distributions.
\end{itemize}
\begin{figure*}
    \centering
    \includegraphics[width=\linewidth]{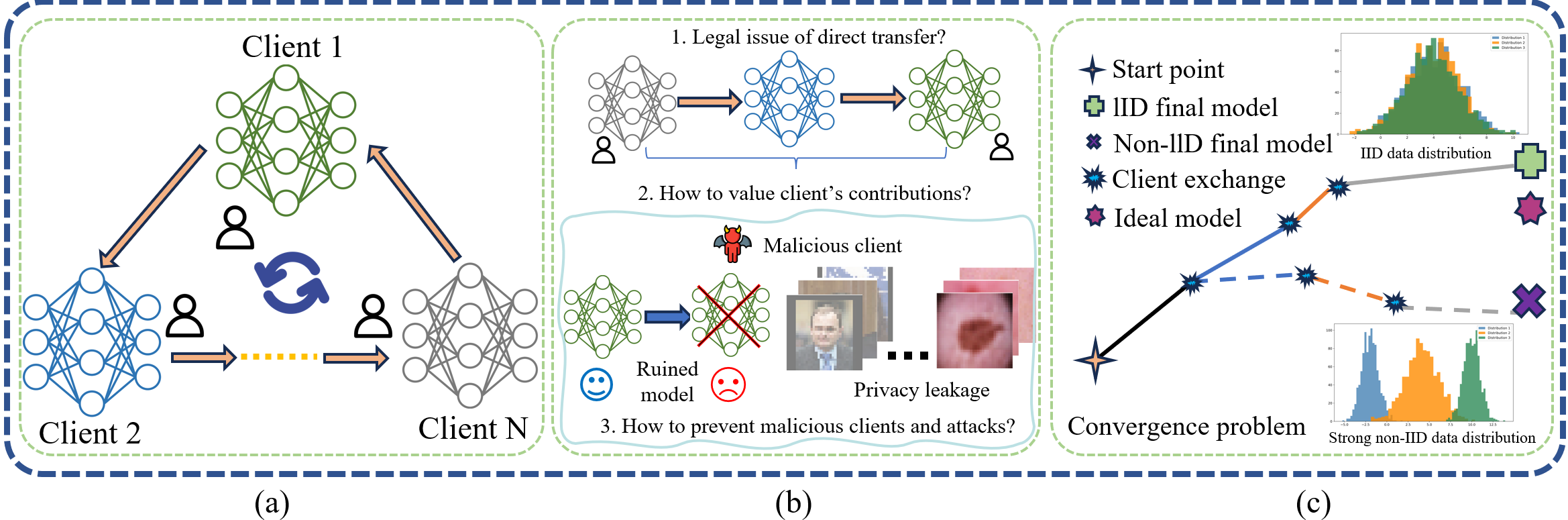}
    \caption{Problem formulation of serial federated learning. (a) Serial federated learning. (b) Privacy risks arise from both malicious clients and gradient leakage attacks during model transfers. (c) Convergence problems caused by data heterogeneity.}
    \label{fig:problem}
\end{figure*}

Given these concerns, privacy preservation, robustness to adversarial or semi-honest behavior, and efficient learning under non-IID conditions emerge as pressing yet intertwined challenges in the deployment of serial FL for healthcare. 

Although some researchers have incorporated various privacy-preserving techniques into serial FL, such as split learning \cite{vepakomma2018split}, differential privacy \cite{yang2023privatefl}, and encryption methods, these approaches still face significant limitations. For example, split learning \cite{vepakomma2018split} divides the model between the client and the server, sharing only intermediate activations and gradients, but remains vulnerable to information leakage from intermediate representations if not adequately protected. PrivateFL \cite{yang2023privatefl} leverages differential privacy to provide formal privacy guarantees, but this often results in substantial degradation of model accuracy, especially in highly heterogeneous data environments. Other frameworks introduce cryptographic protection methods \cite{khan2023split}, such as homomorphic encryption, but always incur high computational costs and complexity, limiting their scalability and practicality in real-world serial FL scenarios.

Regarding fairness, most existing research has concentrated on parallel FL, employing techniques such as Shapley value-based contribution evaluation \cite{wang2020principled,sun2023shapleyfl}. By contrast, fairness in serial FL remains largely underexplored, despite its increased importance due to the sequential nature of model updates, which can amplify the impact of dishonest or low-contributing participants.

To address these issues, we propose TriCon-SF: a novel triple shuffle and contribution-aware serial federated learning framework tailored for segmented healthcare data. Our framework is designed to simultaneously enhance privacy, fairness, and learning performance in cross-silo environments. Specifically, the key contributions are as follows:
\begin{itemize}
    \item We design a triple-shuffle mechanism that randomizes model layers, data segments, and client training sequences to disrupt deterministic training patterns, mitigate inference risks, and reduce linkability.
    \item A segmented training strategy that partitions datasets into variable-sized segments to better balance heterogeneous data distributions across clients and enhance model convergence.
    \item A contribution-aware mechanism to quantify each client’s impact on the global model, enabling the identification of dishonest behavior and promoting accountability in training.
    \item We conduct extensive experiments on real-world non-IID healthcare datasets, demonstrating that TriCon-SF outperforms existing serial and parallel FL methods in terms of accuracy, communication efficiency, and convergence. Security analysis further supports its resilience against client-side privacy attacks.
    
\end{itemize}
The remainder of the paper is organized as follows. Section II reviews the related work, focusing on federated server-based learning methods and federated cross-learning approaches. Section III presents the problem definitions, including the formalization of the serial FL problem, the challenges of non-IID data, and the issue of model gradient leakage. Section IV introduces the proposed methodology, including the system setup, framework overview, and the contribution-aware evaluation mechanism. Section V provides a theoretical analysis of the computational, communication costs, and security of TriCon-SF. Section VI presents a comprehensive experimental evaluation, covering model accuracy, communication cost, and security analysis. Finally, Section VII concludes the paper.

\section{RELATED WORK}

In recent years, a wide range of FL frameworks have been developed to enable privacy-preserving training across multiple institutions, which can be broadly categorized into server-based aggregation methods and cross-learning approaches.

\subsection{Federated server-based learning methods}
The federated server-based weighted aggregation method, introduced by McMahan et al.\cite{mcmahan2017communication} in 2017, is a foundational paradigm that allows clients to collaboratively train a global model without sharing raw data. Participants compute local updates which are then aggregated by a central server using techniques like Federated Averaging. Building on this foundation, numerous extensions have been proposed to address issues such as system heterogeneity and communication inefficiency. Notable examples include FedProx \cite{li2020federated}, which tackles client heterogeneity; SCAFFOLD \cite{karimireddy2020scaffold}, which reduces client drift via control variates; FedDisco \cite{ye2023feddisco}, which introduces distributed optimization for scalability; and FedLabX \cite{yan2024fedlabx}, which provides practical and privacy-aware FL frameworks.

Despite their effectiveness, these server-based methods often fall short in guaranteeing robust privacy. Gradient leakage attacks \cite{wang2020tackling} have demonstrated that even without direct data sharing, sensitive information can be inferred from shared gradients or model parameters. This creates a persistent privacy vulnerability, especially in sensitive domains such as healthcare and finance.

\subsection{Federated cross-learning methods}
To mitigate communication overhead and server dependency, federated cross-learning methods have emerged as alternatives. One of the earliest and most notable is CWT, introduced by Chang et al. \cite{chang2018distributed}, where the model is trained sequentially across institutions. Each client refines the model using its local data before passing it to the next. This approach minimizes infrastructure requirements and reduces communication complexity compared to server-based aggregation.

Variants of CWT, including Incremental Institutional Learning (IIL) and Cyclic Institutional Incremental Learning (CIIL) \cite{sheller2019multi}, continue along this line of research, demonstrating strong performance in low-bandwidth environments. However, direct model transfer between institutions without aggregation can expose models to privacy risks and tampering by semi-honest participants.

Split Learning (SL) and its variants \cite{gupta2018distributed, yuan2024decentralized} operate by partitioning a deep model into multiple sub-models, which are hosted on different entities. These entities perform forward or backward propagation locally and engage in collaborative training by exchanging intermediate activations or gradients. For example, SplitNN \cite{vepakomma2018split} is a multi-party SL framework that demonstrated its privacy-preserving capabilities in healthcare scenarios. STSL \cite{yuan2023peer} models temporal associations across patients through secure access sequence encoding and a scheduling mechanism, enabling privacy-preserving continual learning. SplitFed \cite{thapa2022splitfed} combines the parallel processing capability of FL with the model-splitting architecture of SL to improve performance.

Despite its advantages in privacy and efficiency, SL still faces several key challenges. Although the amount of data transferred in each communication round is small, a full training iteration typically requires highly frequent interactions, leading to significant communication overhead. Furthermore, most SL methods are not optimized for non-IID data, and their performance is highly sensitive to the model's cut point and the data distribution \cite{thapa2022splitfed}. Finally, while model partitioning can reduce raw data exposure to some extent, research has shown that intermediate representations can still leak private information. Without additional protection mechanisms, SL is vulnerable to information inference attacks \cite{zhu2019deep, pasquini2021unleashing, nguyen2023split}.

\section{PROBLEM DEFINITION}

In this section, we formally define the serial FL problem, characterize the non-IID data distribution challenge, and describe the model gradient leakage issue. 

\subsection{Formalization of the Serial FL Problem}
The basic serial FL problem is to minimize a global objective function:
\begin{equation}
    \min _{\theta \in \mathbb{R}^d}\left\{F(\theta):=\frac{1}{n} \sum_{i=1}^n\left(F_i(\theta):=\mathbb{E}_{\xi \sim \mathcal{D}_i}\left[f_i(\theta ; \xi)\right]\right)\right\}.
\end{equation}

where $\mathbf{\theta} \in \mathbb{R}^d$ denotes the model parameters to be optimized. The local objective for client $i$, $F_i(\mathbf{\theta})$, is the expected loss over its local data distribution $\mathcal{D}_i$. For a finite dataset, this is computed as the empirical risk:
\begin{equation}
    F_i(\mathbf{\theta}) = \frac{1}{|\mathcal{D}_i|} \sum_{\xi \in \mathcal{D}_i} f_i(\mathbf{\theta}; \xi),
\end{equation}
where $f_i$ is the loss function and $\xi$ represents a data sample from $\mathcal{D}_i$.

\subsection{Formalization of the Non-IID Problem}
A key challenge in federated learning is the non-IID nature of data across clients. This implies that the local data distribution $\mathcal{D}_i$ of a client $i$ may differ from that of another client $j$. A common form of this heterogeneity is label distribution skew, where the marginal probability of labels varies across clients. This can be formally expressed as:
\begin{equation}
    P_i(y) \neq P_j(y) \quad \text{for some clients } i \neq j,
\end{equation}
where $P_i(y)$ is the probability of observing label $y$ on client $i$. This discrepancy can severely hinder model performance and convergence.

\subsection{The Model Gradient Leakage Problem}
The gradient leakage attack, also known as a data reconstruction attack, poses a significant security threat in serial FL. The core idea is to reconstruct a client's private training data by exploiting the shared model gradients. 

The attack proceeds as follows: an adversary, who has access to the true gradients $\nabla\mathbf{\theta}$ computed from a victim's private data $(\mathbf{x}, \mathbf{y})$, initializes a pair of synthetic data $(\mathbf{x}', \mathbf{y}')$ with random values. These dummy data are the variables to be optimized. The adversary then performs a forward pass using $(\mathbf{x}', \mathbf{y}')$ to compute a synthetic gradient $\nabla\mathbf{\theta}'$. The objective is to minimize the distance between the true gradient $\nabla\mathbf{\theta}$ and the synthetic gradient $\nabla\mathbf{\theta}'$. This is achieved by iteratively updating $(\mathbf{x}', \mathbf{y}')$ via gradient descent. When the optimization converges, the resulting $(\mathbf{x}'^*, \mathbf{y}'^*)$ will closely approximate the victim's original private data.

The optimization problem is formulated as:
\begin{equation}
\begin{aligned}
\mathbf{x}^{\prime *}, \mathbf{y}^{\prime *} & =\underset{\mathbf{x}^{\prime}, \mathbf{y}^{\prime}}{\arg \min }\left\|\nabla \mathbf{\theta}^{\prime}-\nabla \mathbf{\theta}\right\|^2 \\
& =\underset{\mathbf{x}^{\prime}, \mathbf{y}^{\prime}}{\arg \min }\left\|\frac{\partial \mathcal{L}\left(\mathcal{F}\left(\mathbf{x}^{\prime}, \mathbf{\theta}\right), \mathbf{y}^{\prime}\right)}{\partial \mathbf{\theta}}-\nabla \mathbf{\theta}\right\|^2.
\end{aligned}
\label{eq:grad_leakage} 
\end{equation}

\begin{figure*}
    \centering
    \includegraphics[width=\linewidth]{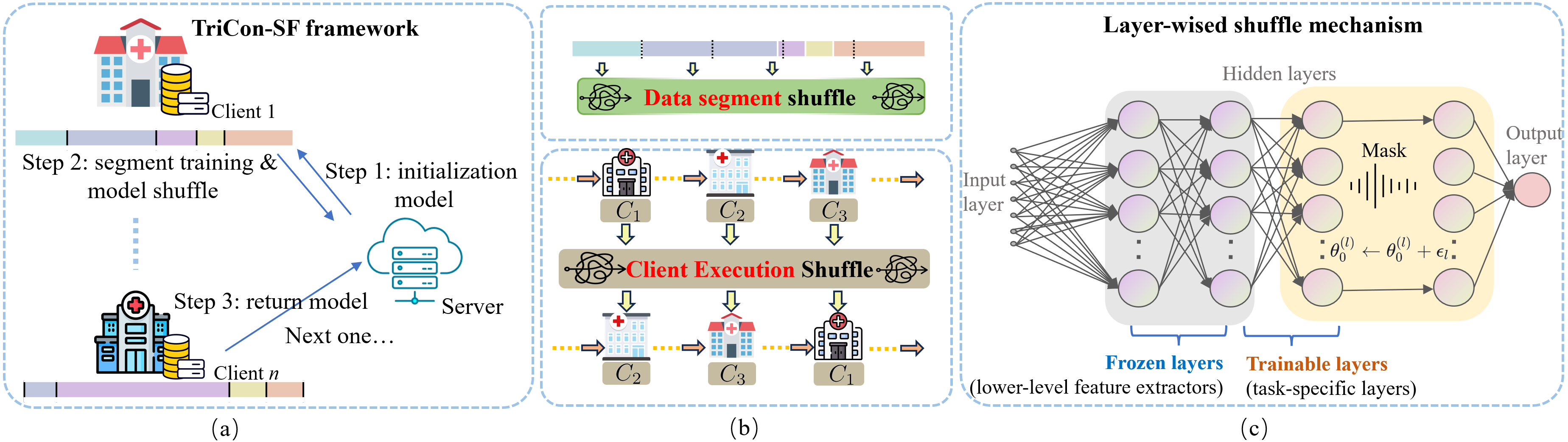}
    \caption{Overview of the TriCon-SF framework. (a) The training pipeline begins with the central server initializing and broadcasting the global model with parameters. Each client performs segment-wise training on locally partitioned data and returns the updated model sequentially. (b) To enhance privacy and robustness, TriCon-SF introduces two shuffling strategies: data segment shuffle, which randomly permutes intra-client data segments, and client execution shuffle, which dynamically alters the client training order across rounds. (c) The Layer-wise Shuffle Mechanism enhances model obfuscation and protects against gradient leakage by randomly masking trainable hidden layers.}
    \label{fig:framework}
\end{figure*}

\section{METHODOLOGY}
In this section, we present the methodology by first describing the system setup and then outlining the overall framework, which is built upon two key mechanisms: the triple-shuffle mechanism and the contribution-aware evaluation mechanism.

\subsection{System setup}
The TriCon-SF framework involves a central server and a set of clients. The central server, denoted by $\mathcal{A}$, is responsible for initializing model parameters and coordinating the training process, while clients perform local model updates.Let there be $n$ clients, represented as $\mathcal{C} = \{C_1, C_2, \ldots, C_n\}$, where each client $C_i$ holds a private dataset $D_i$. The global dataset is thus $\mathcal{D} = \{D_1, D_2, \ldots, D_n\}$, where each $D_i$ is composed of disjoint data segments ${D_i^j}$ that are locally stored and not shared. Table~\ref{tbl:para} summarizes the notations and parameters used throughout our framework.

From the threat model, we assume a semi-honest central server that complies with the protocol while potentially attempting to infer private information from the received data. The clients are either semi-honest or malicious. Some clients may tamper with model parameters, carry out inference attacks, or exploit the training process without making meaningful contributions. The communication between the server and the clients is considered secure, with encrypted channels that prevent interception by external adversaries.

\begin{table}[h!]
\centering
\caption{The list of parameters}
\label{tbl:para}
\setlength{\tabcolsep}{2mm}
\begin{tabular}{cc} 
\toprule
\textbf{Parameter} & \textbf{Definition}\\ 
\hline
$\mathcal{C}$ & Set of all participating clients, $(C_1, C_2, \ldots, C_n)$ \\
$n$ & Total number of clients  \\
$S$ & A subset of clients, $S \subseteq \mathcal{C}$ \\
$D_i$ & Local dataset of client $\mathcal{C}_i$ \\
$l_i = |D_i|$ & Size of local dataset for client $\mathcal{C}_i$ \\
$k$ & Number of segments each dataset is divided into \\
$s_{\text{min}}$ & Minimum size of each data segment \\
$D_{i,j}$ & $j$-th data segment of client $\mathcal{C}_i$ \\
$\mathcal{A}$ & Federated learning server \\
$\mathcal{C}_i$ & The $i$-th client in the federation \\
$\mathcal{M}_0$ & Initial global model \\
$\theta$ & Model parameters of the current model \\
$\theta_0$ & Initial model parameters \\
$\theta^{(l)}$ & Parameters of the $l$-th layer in the model \\
$\theta_{\text{train}}$ & Trainable subset of model parameters \\
$\theta_{\text{freeze}}$ & Frozen subset of model parameters \\
$L$ & Total number of layers in the model \\
$\mathcal{L}_{\text{perturb}}$ & Subset of layers selected for perturbation \\
$\epsilon_l$ & Gaussian noise added to layer $l$: $\epsilon_l \sim \mathcal{N}(0, \sigma^2 I)$ \\
$\pi$ & Random permutation of client indices (client order shuffle) \\
$S_i^{\pi_m}$ & Set of clients preceding $C_i$ in permutation $\pi_m$ \\
$\eta$ & Local learning rate for SGD \\
$\tau$ & Number of local SGD steps per client update \\
$\mathcal{B}$ & Mini-batch sampled from $D_{i,j}$ \\
$f_i(\theta; \xi)$ & Local loss function on data sample $\xi$ for client $i$ \\
$t$ & Current iteration or update index \\
$\phi_i$ & Shapley value: marginal contribution of client $C_i$ \\
$\hat{\phi}_i$ & Estimated Shapley value of client $C_i$ \\
$\phi_{\text{min}}$ & Threshold for minimum acceptable contribution \\
$\text{status}(C_i)$ & Classification result of client $C_i$ \\
$m$ & Number of Monte Carlo permutations used for estimation \\

\bottomrule
\end{tabular}
\end{table} 

\subsection{Framework overview}
The proposed TriCon-SF framework operates in a serial FL setting, tailored for privacy-preserving training over heterogeneous and segmented datasets. The overall workflow is structured into two primary phases: initialization and training, underpinned by a novel triple-shuffle mechanism to ensure data confidentiality and resistance against gradient leakage. The framework overview can be found in Figure \ref{fig:framework}, and the main algorithm can be found in Algorithm \ref{alg:tricon-sf}.

\subsubsection{Initialization phase}
During the initialization phase, the server $\mathcal{A}$ determines two key parameters: the number of data segments $k \in \mathbb{N}$ and the minimum segment size $s_{\text{min}} \in \mathbb{N}$ in such a way that:
\begin{equation}
    \min_{i \in\{1, \ldots, n\}} |D_i| \geq k \cdot s_{\min}.
\end{equation}
This ensures each client has sufficient data for segment-wise training. The parameters $k$ and $s_{\text{min}}$ are then broadcast to all clients. Each client $\mathcal{C}_i$ subsequently partitions its local dataset into $k$ non-overlapping segments: 
\begin{equation}
\begin{array}{r}
D_i = (D_{i, 1}, D_{i, 2}, \ldots, D_{i,k}), \\
\text { where }\left|D_{i, j}\right| \geq s_{\min }, \forall j \in\{1, \ldots, k\}.
\end{array}
\end{equation}
This \textbf{data segmentation shuffle} constitutes the \textbf{first layer} of the proposed triple-shuffle mechanism, promoting intra-client randomness and granular privacy at the data level.

To further enhance privacy and reduce the risk of gradient leakage at initialization, we introduce layer-wise partial perturbation. Let the initial model be $\mathcal{M}_0$ with parameters $\theta_0 = \{\theta_0^{(1)}, \ldots, \theta_0^{(L)}\}$ across $L$ layers. A random subset $\mathcal{L}_{\text{perturb}} \subseteq \{1, \ldots, L\}$ of layers is selected for perturbation. For each $l \in \mathcal{L}_{\text{perturb}}$, additive Gaussian noise $\epsilon_l \sim \mathcal{N}(0, \sigma^2 I)$ is applied:
\begin{equation}
\theta_0^{(l)} \leftarrow \theta_0^{(l)}+\epsilon_l, \quad l \in \mathcal{L}_{\text {perturb }}.
\end{equation}
This \textbf{model layer shuffle} forms the \textbf{second layer} of the triple-shuffle mechanism, obfuscating model initialization across clients. The perturbed model is then distributed to the clients for the training phase.

\subsubsection{Training phase}
At the start of each training round, the server $\mathcal{A}$ samples a random permutation $\pi \in \mathcal{S}_n$ over client indices, defining the order of updates:
\begin{equation}
    C_{\pi(1)} \rightarrow C_{\pi(2)} \rightarrow \cdots \rightarrow C_{\pi(n)}.
\end{equation}
The randomized execution sequence $\pi$ is resampled in each round, ensuring that no fixed client order is established across iterations. This \textbf{client execution shuffle}, resampled per round, constitutes the \textbf{third layer} of the Triple-Shuffle Mechanism, preventing fixed communication paths and improving resistance to adversarial model inversion.

Each client $C_i$ receives the current model parameters $\theta^t$ and selects one of its $k$ data segments $D_{i, j}$ for training. Let $f_i(\theta; \xi)$ denote the local loss function for data sample $\xi \in D_\{i, j\}$. The client performs $\tau$ steps of stochastic gradient descent (SGD) over the segment, updating only $\theta_{\text{train}}$ as follows:
\begin{equation}
\theta_{\text {train }}^{t+1} \leftarrow \theta_{\text {train }}^t-\eta \cdot \frac{1}{|\mathcal{B}|} \sum_{\xi \in \mathcal{B}} \nabla_{\theta_{\text {train }}} f_{\pi(i)}\left(\theta^t ; \xi\right).
\end{equation}

where $\mathcal{B}$ is a mini-batch sampled from $D_{i, j}$, and $\eta$ is the learning rate. In each model, we partition $\theta$ into two subsets:
\begin{itemize}
    \item Trainable layers: $\theta_{\text{train}} \subseteq \theta$ (typically task-specific or higher layers)
    \item Frozen layers: $\theta_{\text{freeze}} = \theta \setminus \theta_{\text{train}}$ (e.g., lower-level feature extractors or backbone).
\end{itemize}

During local updates, only $\theta_{\text{train}}$ is updated locally and $\theta_{\text{freeze}}$ remains unchanged. The updated model $\theta^{t+1}$ is then passed to the next client in the permutation. 

Once local training is complete, client $C_i$ forwards the updated model $\theta^{t+1}$ to the next client $C_{i+1}$ in the permutation sequence. This process continues for $n \cdot k$ total updates (assuming each client updates once per segment), ensuring full coverage of the distributed dataset while maintaining serial training.

\begin{algorithm}
\caption{TriCon-SF: Triple-shuffle serial federated learning framework}
\label{alg:tricon-sf}
\KwIn{Number of clients $n$; each client $\mathcal{C}_i$ with local dataset $D_i$; initial model $\mathcal{M}_0$ with parameters $\boldsymbol{\theta}_0$; learning rate $\eta$; local steps $\tau$; segment count $k$; minimum segment size $s_{\text{min}}$; Gaussian noise scale $\sigma$; total global rounds $T$.}
\KwOut{Trained global model $\boldsymbol{\theta}$}

\textbf{Initialization Phase:} \\
\ForEach{client $\mathcal{C}_i$}{
  Verify $|D_i| \geq k \cdot s_{\text{min}}$ \\
  Partition $D_i$ into $k$ disjoint segments $D_{i,1}, \ldots, D_{i,k}$ with $|D_{i,j}| \geq s_{\text{min}}$ \tcp*{Data segment shuffle}
}

Randomly select $\mathcal{L}_{\text{perturb}} \subseteq \{1, \ldots, L\}$ \tcp*{Layer-wise shuffle}
\ForEach{$l \in \mathcal{L}_{\text{perturb}}$}{
  $\boldsymbol{\theta}_0^{(l)} \leftarrow \boldsymbol{\theta}_0^{(l)} + \epsilon_l$, where $\epsilon_l \sim \mathcal{N}(0, \sigma^2 I)$
}

\vspace{1mm}
\textbf{Training Phase:} \\
Initialize model: $\boldsymbol{\theta} \leftarrow \boldsymbol{\theta}_0$ \tcp*{Model is initialized only ONCE before all rounds}

\For{$t = 1$ to $T$}{
  Server samples a random permutation $\pi_t$ over client indices $\{1, \ldots, n\}$ \tcp*{Client sequence shuffle}

  \For{$i = 1$ to $n$}{
    Client $\mathcal{C}_{\pi_t(i)}$ receives the current model $\boldsymbol{\theta}$ \\
    Randomly select a data segment $D_{\pi_t(i), j}$ 

    Freeze $\boldsymbol{\theta}_{\text{freeze}}$, update only $\boldsymbol{\theta}_{\text{train}}$ \\
    
    \For{$\tau$ local steps}{
      Sample mini-batch $\mathcal{B} \subseteq D_{\pi_t(i), j}$ \\
      $\boldsymbol{\theta}_{\text{train}} \leftarrow \boldsymbol{\theta}_{\text{train}} - \eta \cdot \frac{1}{|\mathcal{B}|} \sum\limits_{\xi \in \mathcal{B}} \nabla_{\boldsymbol{\theta}_{\text{train}}} f_{\pi_t(i)}(\boldsymbol{\theta}; \xi)$
    }
    
    Pass the updated model $\boldsymbol{\theta}$ to the next client in the sequence $\pi_t$
  }

}
\Return Final model $\boldsymbol{\theta}$
\end{algorithm}

\subsection{Contribution-aware evaluation mechanism}
To ensure fairness and integrity in collaborative training, we integrate a contribution-aware evaluation mechanism that quantifies each client's impact on the global model performance. This mechanism is built upon the Shapley value framework, a cooperative game-theoretic approach that assigns a fair value of contribution to each participant.
\subsubsection{Shapley value estimation}
Let $\mathcal{V}: 2^{\mathcal{C}} \rightarrow \mathbb{R}$ be a utility function that measures model performance (e.g., accuracy, loss reduction) obtained by aggregating updates from a given subset of clients $S \subseteq \mathcal{C}$. The Shapley value $\phi_i$ for client $C_i$ is defined as:
\begin{equation}
\phi_i=\sum_{S \subseteq \mathcal{C} \backslash\left\{C_i\right\}} \frac{|S|!(n-|S|-1)!}{n!}\left[\mathcal{V}\left(S \cup\left\{C_i\right\}\right)-\mathcal{V}(S)\right].
\end{equation}
This value quantifies the marginal contribution of $C_i$ to every possible coalition $S$. Intuitively, it evaluates how much model utility increases when $C_i$ is included in a coalition.

Due to the exponential complexity of computing exact Shapley values, we adopt an approximate estimation using Monte Carlo sampling: for $m$ random permutations $\pi_1, \ldots, \pi_m$ of the clients, we estimate $\phi_i$ as:
\begin{equation}
\hat{\phi}_i=\frac{1}{m} \sum_{m=1}^m\left[\mathcal{V}\left(S_i^{\pi_m} \cup\left\{C_i\right\}\right)-\mathcal{V}\left(S_i^{\pi_m}\right)\right].
\end{equation}

where $S_i^{\pi_m}$ denotes the set of clients preceding $C_i$ in the $m$-th permutation $\pi_m$.

\subsubsection{Free-rider and malicious client detection}
To filter out non-contributing or adversarial behavior, we define a contribution threshold $\phi_{\text{min}}$. After model training, each client's Shapley score $\hat{\phi}_i$ is compared against this threshold:
\begin{equation}
\operatorname{status}\left(C_i\right)= \begin{cases}\text { Honest }, & \text { if } \hat{\phi}_i \geq \phi_{\min } \\ \text { Free-rider or Malicious, } & \text { otherwise }\end{cases}
\end{equation}

Clients whose contributions fall below the threshold are flagged for exclusion or further inspection. This step helps enforce accountability and robust aggregation, preventing free-riders from benefiting without contributing and limiting the impact of malicious manipulation.

\section{THEORETICAL ANALYSIS}
In this section, we present the computation and communication models of the TriCon-SF framework to provide a clearer understanding of its training pipeline. In addition, we include a theoretical analysis of its security guarantees.

\subsection{Communication Model for TriCon-SF}

In the serial architecture of TriCon-SF, communication occurs sequentially as the model is passed from one client to the next, forming a communication chain. The primary communication cost in this process is determined by the size of the model parameters being transmitted.

Let $M_{\text{size}}$ be the size of the model parameters $\mathbf{\theta}$ in bits. The transmission rate for client $i$, denoted as $r_i$, can be modeled by the Shannon-Hartley theorem:
\begin{equation}
    r_i = B \log_2 \left(1 + \frac{\rho_i h_i}{N_0}\right),
\end{equation}
where $B$ is the channel bandwidth, $\rho_i$ is the transmission power of client $i$, $h_i$ is its channel gain, and $N_0$ is the background noise power spectral density.

The time required for a single transmission of the model from a client (e.g., client $i$) to the next in the sequence is:
\begin{equation}
    T_i^{\text{com}} = \frac{M_{\text{size}}}{r_i}.
\end{equation}

A full global round in TriCon-SF involves the model circulating through all $n$ participating clients. This requires a total of $n$ sequential transmissions. Therefore, the total communication time for one complete round, $T_{\text{round}}^{\text{com}}$, is the sum of the individual transmission times along the communication path defined by the permutation $\pi$:
\begin{equation}
    T_{\text{round}}^{\text{com}} = \sum_{i=1}^{n} T_{\pi(i)}^{\text{com}} = \sum_{i=1}^{n} \frac{M_{\text{size}}}{r_{\pi(i)}}.
\end{equation}

\subsection{Security analysis}
Gradient leakage attacks aim to reconstruct private training data by analyzing gradients exchanged during federated learning. To counter this, TriCon-SF integrates a layer shuffle-and-perturb mechanism combining layer-wise shuffling, $L_2$ norm clipping, and Gaussian noise injection, to ensure privacy even under gradient exposure.

\begin{theorem}[Privacy protection via layer shuffle-and-perturb mechanism]\label{thm:shuffle_privacy}
Let $\mathbf{z}$ denote a local feature representation transmitted from a client during training. Under the layer shuffle mechanism, the adversary cannot reconstruct the original input $\mathbf{x}$ with high confidence, even with full access to model updates or gradients.
\end{theorem}

\begin{proof}
TriCon-SF applies the following transformations to each intermediate representation $\mathbf{z}_i$ computed by client $C_i$:
\begin{itemize}
    \item \textbf{Layer-wise shuffling}: A client-specific random permutation $\pi_i$ is applied to the intermediate feature vector:
    \begin{equation}
\tilde{\mathbf{z}}_i=\pi_i\left(\mathbf{z}_i\right)
    \end{equation}
    where $\pi_i$ is unknown to the server. This operation breaks the alignment of semantic features across training steps, making the inverse mapping from gradients to input non-trivial and non-deterministic.
    \item \textbf{$L_2$-norm clipping}: The permuted vector $\tilde{\mathbf{z}}_i$ is then projected onto an $L_2$-norm ball with radius $c$:
    \begin{equation}
        \mathbf{z}_i^{\mathrm{clip}}=\tilde{\mathbf{z}}_i \cdot \min \left(1, \frac{c}{\left\|\tilde{\mathbf{z}}_i\right\|_2}\right)
    \end{equation}
    where $\mathrm{c}=1.0$ in our experiments. This limits the contribution of each data point to the gradient, effectively bounding its sensitivity:
    \begin{equation}
        \Delta f=\max _{\mathbf{x}, \mathbf{x}^{\prime}}\left\|f(\mathbf{x})-f\left(\mathbf{x}^{\prime}\right)\right\|_2 \leq c
    \end{equation}

    \item \textbf{Gaussian noise injection}: To achieve differential privacy, we add noise sampled from a zero-mean Gaussian distribution with standard deviation $\sigma$:
    \begin{equation}
    \hat{\mathbf{z}}_i=\mathbf{z}_i^{\mathrm{clip}}+\mathcal{N}\left(0, \sigma^2 \mathbf{I}\right)
    \end{equation}
with $\sigma = 1.0$ as the noise scale. This perturbation ensures that the probability of observing a specific gradient is nearly the same for neighboring inputs, satisfying $(\varepsilon, \delta)$-differential privacy.
    \item \textbf{Gradient update}: The perturbed features $\hat{\mathbf{z}}_i$ are then used for local gradient computation:
    \begin{equation}
        \nabla_{\hat{\mathbf{z}}} \mathcal{L}=\frac{\partial \mathcal{L}\left(\hat{\mathbf{z}}_i, y_i\right)}{\partial \hat{\mathbf{z}}_i} .
    \end{equation}
    Even if an adversary attempts to perform inversion using the observed gradient difference:
\begin{equation}
    \delta \mathbf{\theta}_k = \mathbf{\theta}_{k+1} - \mathbf{\theta}_k = \eta \cdot \frac{\partial \mathcal{L}}{\partial \mathbf{\theta}}.
\end{equation}
    The randomness induced by shuffling and noise addition makes the inversion problem underdetermined and statistically unreliable.
\end{itemize}
\end{proof}
Our experimental analysis further validates this proof.

\section{EXPERIMENT EVALUATION}
The experimental evaluation section includes the experimental setup, accuracy assessment, computational evaluation, and security analysis.
\subsection{Experimental settings}
\subsubsection{Datasets}
To evaluate the generalization capability of the proposed TriCon-SF framework across heterogeneous data distributions, particularly in the healthcare domain, we conducted experiments on five publicly available cancer diagnosis datasets spanning both tabular and image modalities:

\begin{itemize}
\item \textbf{CUMIDA collection datasets (tabular)}: This collection includes three datasets. The Leukemia\_GSE28497 dataset~\cite{alabdulqader2024improving} supports binary classification of leukemia subtypes and consists of approximately 12,600 gene expression features per sample. The Brain\_GSE50161 dataset~\cite{zhang2020bioinformatics} is designed for multiclass brain tumor classification based on microarray gene expression profiles, enabling the identification of distinct tumor subtypes. The Breast\_GSE45827 dataset~\cite{Feltes2019} includes around 22,000 probe features per sample and is used for predicting breast cancer subtypes from gene expression data.

\item \textbf{PCAWG (Pan-Cancer Analysis of Whole Genomes) (sparse vector)} \cite{jiao2020deep}: A whole-genome sequencing dataset consisting of 2,658 tumor samples, used for multiclass classification across 24 primary cancer types. Input features include somatic mutation distributions and mutation-type profiles represented in sparse vector form.

\item \textbf{HAM10000 (image)} \cite{tschandl2018ham10000}: A dermatoscopic image dataset with 10,015 high-resolution skin lesion samples across seven diagnostic categories. It presents a significant class imbalance, with over 60\% of samples labeled as melanocytic nevi (nv), posing challenges for unbiased model training.
\end{itemize}

\subsubsection{Baselines}
To evaluate the effectiveness of the proposed TriCon-SF framework, we compare its performance against a suite of representative distributed learning methods. These baselines include several widely adopted parallel FL algorithms that are specifically designed to address challenges associated with non-IID data distributions:

\begin{itemize}
\item \textbf{FedAvg}~\cite{mcmahan2017communication}: A foundational FL algorithm that performs client-side local training followed by global aggregation via parameter averaging.

\item \textbf{FedProx}~\cite{li2020federated}: An enhancement of FedAvg that incorporates a proximal term to mitigate client drift by constraining local updates. The proximal penalty coefficient $\mu$ is set to $0.01$, following standard practice.

\item \textbf{SCAFFOLD}~\cite{karimireddy2020scaffold}: A control-variates-based approach that reduces variance in local updates by correcting client drift using server- and client-side control variates.

\item \textbf{FedDisco}~\cite{ye2023feddisco}: A decoupled FL framework that introduces dual control mechanisms to improve convergence under both system and statistical heterogeneity. We adopt the recommended hyperparameter settings: $a = 0.5$ and $b = 0.1$.

\end{itemize}

In addition, we include two serial FL variants based on the SFL paradigm~\cite{thapa2022splitfed}:

\begin{itemize}
    \item \textbf{SFLV1:} A hybrid approach where the client-side model is trained locally, and server-side model updates are aggregated in parallel across clients.
    \item \textbf{SFLV2}: A sequential variant of SplitFed in which the server-side model is updated client-by-client, introducing a serial update flow that aligns more closely with TriCon-SF.
\end{itemize}

\begin{table*}[!t]
\centering
\caption{Test accuracy (\%) under different client numbers and non-IID degree on CUMIDA datasets.}
\footnotesize
\setlength{\tabcolsep}{1pt}
\renewcommand{\arraystretch}{0.85}
\label{tab:test_acc_cumida_final_transposed1}
\begin{tabularx}{0.9\textwidth}{>{\centering\arraybackslash}X|*{3}{>{\centering\arraybackslash}X}|*{3}{>{\centering\arraybackslash}X}|*{3}{>{\centering\arraybackslash}X}@{}}
\toprule
\textbf{Clients} & \multicolumn{3}{c|}{\textbf{2}} & \multicolumn{3}{c|}{\textbf{5}} & \multicolumn{3}{c}{\textbf{10}} \\
\midrule
$\beta$ & 0.5 & 1 & 10 & 0.5 & 1 & 10 & 0.5 & 1 & 10 \\
\midrule
\multicolumn{10}{c}{\cellcolor{lightgray}\textbf{\textit{Leukemia}} - Centralized ACC $=88.5\pm1.6$} \\
\midrule
\cellcolor{lightgray}FedAvg        & 78.2$\pm$4.3 & 81.6$\pm$1.6 & 81.6$\pm$1.6 & 79.3$\pm$2.8 & 80.5$\pm$3.3 & 81.6$\pm$1.6 & 73.6$\pm$4.3 & 74.7$\pm$4.3 & 80.5$\pm$3.3 \\
\cellcolor{lightgray}FedProx       & 77.0$\pm$5.9 & 82.8$\pm$5.6 & 83.9$\pm$3.3 & 79.3$\pm$4.9 & 79.3$\pm$4.9 & 81.6$\pm$1.6 & 74.7$\pm$3.3 & 78.2$\pm$3.3 & 79.3$\pm$2.8 \\
\cellcolor{lightgray}Scaffold      & 75.9$\pm$2.8 & 82.8$\pm$2.8 & 83.9$\pm$1.6 & 77.0$\pm$4.3 & 75.9$\pm$5.6 & 79.3$\pm$0.0 & 74.7$\pm$5.9 & 77.0$\pm$1.6 & 81.6$\pm$1.6 \\
\cellcolor{lightgray}FedDisco      & \textcolor{blue}{\textbf{81.6$\pm$3.3}} & \textcolor{blue}{\textbf{83.9$\pm$1.6}} & 82.8$\pm$2.8 & 74.7$\pm$1.6 & 79.3$\pm$4.9 & 82.8$\pm$5.6 & 72.4$\pm$5.6 & 78.2$\pm$7.1 & 80.5$\pm$1.6 \\
\cellcolor{lightgray}SFLV1         & 77.3$\pm$1.8 & 74.0$\pm$1.0 & 82.2$\pm$0.9 & \textcolor{blue}{\textbf{85.5$\pm$0.0}} & 73.1$\pm$1.7 & 76.1$\pm$0.0 & 58.2$\pm$4.9 & 58.2$\pm$8.4 & 68.0$\pm$9.0 \\
\cellcolor{lightgray}SFLV2         & 78.3$\pm$1.9 & 74.5$\pm$0.5 & 83.2$\pm$1.3 & 83.6$\pm$0.0 & 74.6$\pm$0.7 & 75.7$\pm$1.0 & 64.4$\pm$4.7 & 66.0$\pm$3.1 & 77.4$\pm$4.0 \\
\cellcolor{lightgray}\textbf{TriCon-SF} & 80.5$\pm$1.6 & 80.5$\pm$1.6 & \textcolor{blue}{\textbf{85.1$\pm$1.6}} & 80.5$\pm$1.6 & \textcolor{blue}{\textbf{81.6$\pm$1.6}} & \textcolor{blue}{\textbf{85.1$\pm$3.3}} & \textcolor{blue}{\textbf{81.6$\pm$1.6}} & \textcolor{blue}{\textbf{82.8$\pm$2.8}} & \textcolor{blue}{\underline{\textbf{87.4$\pm$1.6}}} \\
\midrule
\multicolumn{10}{c}{\cellcolor{lightgray}\textbf{\textit{Breast}} - Centralized ACC $=95.8\pm2.9$} \\
\midrule
\cellcolor{lightgray}FedAvg        & 85.4$\pm$5.9 & 85.4$\pm$2.9 & 89.6$\pm$10.6 & 70.8$\pm$7.8 & 75.0$\pm$5.1 & 87.5$\pm$5.1 & 66.7$\pm$12.8 & 70.8$\pm$7.8 & 79.2$\pm$2.9 \\
\cellcolor{lightgray}FedProx       & 83.3$\pm$2.9 & \textcolor{blue}{\textbf{89.6$\pm$2.9}} & 87.5$\pm$5.1 & 75.0$\pm$10.2 & 79.2$\pm$11.8 & 85.4$\pm$5.9 & 72.9$\pm$16.4 & 75.0$\pm$10.2 & 83.3$\pm$10.6 \\
\cellcolor{lightgray}Scaffold      & 83.3$\pm$15.6 & 83.3$\pm$7.8 & 89.6$\pm$7.8 & 81.2$\pm$5.1 & 83.3$\pm$2.9 & 85.4$\pm$2.9 & 72.9$\pm$5.9 & 81.2$\pm$5.1 & 85.4$\pm$7.8 \\
\cellcolor{lightgray}FedDisco      & 79.2$\pm$15.6 & 87.5$\pm$8.8 & \textcolor{blue}{\underline{\textbf{93.8$\pm$5.1}}} & 79.2$\pm$20.6 & 83.3$\pm$7.8 & 85.4$\pm$5.9 & 70.8$\pm$15.6 & 72.9$\pm$11.8 & 83.3$\pm$2.9 \\
\cellcolor{lightgray}SFLV1         & 92.0$\pm$2.0 & 80.2$\pm$1.8 & 72.1$\pm$5.6 & 69.8$\pm$6.2 & 70.0$\pm$5.1 & 61.6$\pm$8.6 & 56.4$\pm$0.0 & 63.9$\pm$3.7 & 65.5$\pm$3.7 \\
\cellcolor{lightgray}SFLV2         & \textcolor{blue}{\textbf{92.3$\pm$2.1}} & 81.0$\pm$3.1 & 75.8$\pm$2.8 & 66.1$\pm$6.2 & 73.0$\pm$4.1 & 62.2$\pm$8.7 & 56.4$\pm$0.0 & 65.9$\pm$7.8 & 69.0$\pm$4.4 \\
\cellcolor{lightgray}\textbf{TriCon-SF} & 83.3$\pm$5.9 & 83.3$\pm$5.9 & 85.4$\pm$5.9 & \textcolor{blue}{\textbf{87.5$\pm$5.1}} & \textcolor{blue}{\textbf{87.5$\pm$0.0}} & \textcolor{blue}{\underline{\textbf{93.8$\pm$5.1}}} & \textcolor{blue}{\textbf{84.4$\pm$3.1}} & \textcolor{blue}{\textbf{89.6$\pm$2.9}} & \textcolor{blue}{\textbf{91.7$\pm$2.9}} \\
\midrule
\multicolumn{10}{c}{\cellcolor{lightgray}\textbf{\textit{Brain}} - Centralized ACC $=92.3\pm0.0$} \\
\midrule
\cellcolor{lightgray}FedAvg        & 79.5$\pm$7.3 & 79.5$\pm$3.6 & 84.6$\pm$0.0 & 76.9$\pm$6.3 & 82.1$\pm$9.6 & 87.2$\pm$3.6 & 64.1$\pm$7.3 & 76.9$\pm$6.3 & 82.1$\pm$3.6 \\
\cellcolor{lightgray}FedProx       & 71.8$\pm$3.6 & 84.6$\pm$0.0 & 87.2$\pm$3.6 & 71.8$\pm$18.1 & 79.5$\pm$7.3 & 82.1$\pm$9.6 & 64.1$\pm$9.6 & 71.8$\pm$18.1 & 76.9$\pm$12.6 \\
\cellcolor{lightgray}Scaffold      & 82.1$\pm$7.3 & 87.2$\pm$3.6 & 84.6$\pm$6.3 & 71.8$\pm$7.3 & 76.9$\pm$10.9 & 82.1$\pm$7.3 & 71.8$\pm$14.5 & 74.4$\pm$9.6 & 84.6$\pm$10.9 \\
\cellcolor{lightgray}FedDisco      & 76.9$\pm$10.9 & 79.5$\pm$3.6 & 84.6$\pm$6.3 & 79.5$\pm$3.6 & 84.6$\pm$10.9 & 82.1$\pm$7.3 & 79.5$\pm$9.6 & 84.6$\pm$6.3 & 84.6$\pm$6.3 \\
\cellcolor{lightgray}SFLV1         & 65.9$\pm$0.0 & 78.2$\pm$2.5 & 86.4$\pm$2.8 & 79.1$\pm$0.0 & 83.2$\pm$3.6 & 86.9$\pm$0.0 & 64.7$\pm$3.5 & 63.3$\pm$0.0 & 76.7$\pm$5.6 \\
\cellcolor{lightgray}SFLV2         & 65.9$\pm$0.0 & 81.5$\pm$4.2 & 85.6$\pm$2.7 & 79.1$\pm$0.0 & 85.3$\pm$0.0 & 86.9$\pm$0.0 & 66.3$\pm$3.9 & 63.3$\pm$0.0 & 80.4$\pm$0.0 \\
\cellcolor{lightgray}\textbf{TriCon-SF} & \textcolor{blue}{\textbf{87.2$\pm$3.6}} & \textcolor{blue}{\textbf{89.7$\pm$3.6}} & \textcolor{blue}{\underline{\textbf{92.3$\pm$0.0}}} & \textcolor{blue}{\textbf{80.8$\pm$3.8}} & \textcolor{blue}{\textbf{89.7$\pm$3.6}} & \textcolor{blue}{\textbf{89.7$\pm$3.6}} & \textcolor{blue}{\textbf{89.7$\pm$3.6}} & \textcolor{blue}{\textbf{88.5$\pm$3.8}} & \textcolor{blue}{\underline{\textbf{92.3$\pm$0.0}}} \\
\midrule
\addlinespace[2pt]
\multicolumn{10}{p{\dimexpr0.9\textwidth-2\tabcolsep\relax}}{\textbf{Note:} Higher accuracy is better. The best accuracy in each column is marked in \textbf{\textcolor{blue}{blue}}. The overall best result for each dataset is \underline{underlined}.} \\
\bottomrule
\end{tabularx}
\end{table*}

\subsubsection{Experimental setting}
The training, validation, and test datasets for each client were split in a stratified manner according to the local class distribution, with a ratio of 8:1:1. We employed a multilayer perceptron model for the CUMIDA and PCAWG datasets, while a ResNet18~\cite{he2016deep} model was applied for the HAM10000 dataset. The Adam optimizer~\cite{kingma2014adam} was adopted for all training processes. The experiments were executed on a server equipped with 2X Intel(R) Xeon(R) Platinum 8358 CPUs @ 2.60GHz and 1 NVIDIA L40S GPU. A detailed parameter setting can be found in Table \ref{tab:training_params}.

\begin{table}[]
    \centering
        \caption{Training parameters including the size of hidden layers, dropout rates, weight decay, batch sizes, and learning rates for each dataset used in our experiments.}
    \label{tab:training_params}
    \scalebox{0.9}{
    \begin{tabular}{lcccccc}
\toprule
\textbf{} & \textbf{Hidden} &  & \textbf{Weight} & \textbf{Batch} & \textbf{Learning} \\ 
\textbf{} & \textbf{layers} & \textbf{Dropout} & \textbf{decay} & \textbf{size} & \textbf{rate}\\ \midrule
Breast & 10 & 0.1 & 1e-1 & \multirow{3}{*}{32} & \multirow{3}{*}{2e-5} \\
Brain & 64 & 0.5 & 1e-2 & & \\
Leukemia & 64 & 0.2 & 5e-4 & & \\ \midrule
PCAWG & {[}2048,1024,512{]} & 0.3 & 1e-3 & 32 & \multirow{2}{*}{1e-4} \\
HAM10000 & - & - & 0 & 256 & \\ \bottomrule
\end{tabular}}
\end{table}

\subsubsection{Federated setting}
To simulate potential heterogeneous data distribution among clients in FL, we distributed data to clients according to a Dirichlet distribution $Dir_\beta$ \cite{li2021model}, where a smaller $\beta$ value leads to a more dispersed distribution of data. We varied $\beta$ in $[0.5, 1, 10]$. Given the dataset volumes, the number of clients was set to $[2, 5, 10]$ for the CUMIDA datasets and $[5, 10, 20]$ for the PCAWG and HAM10000 datasets. The number of communication rounds in FL was set to 60, 100, and 200 rounds for the CUMIDA, PCAWG, and HAM10000 datasets, respectively, with one local epoch per communication round for all algorithms. 

\begin{table*}[htb]
\centering
\caption{Test accuracy (\%) under different client numbers and non-IID degree on PCAWG and HAM10000 datasets.}
\label{tab:test_acc_cumida_final_transposed2}
\footnotesize
\setlength{\tabcolsep}{1pt}
\renewcommand{\arraystretch}{0.85}
\begin{tabularx}{0.9\textwidth}{>{\centering\arraybackslash}X|*{3}{>{\centering\arraybackslash}X}|*{3}{>{\centering\arraybackslash}X}|*{3}{>{\centering\arraybackslash}X}@{}}
\toprule
\textbf{Clients} & \multicolumn{3}{c|}{\textbf{5}} & \multicolumn{3}{c|}{\textbf{10}} & \multicolumn{3}{c}{\textbf{20}} \\
\midrule
$\beta$ & 0.5 & 1 & 10 & 0.5 & 1 & 10 & 0.5 & 1 & 10 \\
\midrule

\multicolumn{10}{c}{\cellcolor{lightgray}\textbf{\textit{PCAWG}} - Centralized training ACC $=88.5\pm0.3$} \\
\midrule
\cellcolor{lightgray}FedAvg        & 82.7$\pm$1.3 & 83.7$\pm$0.5 & 84.9$\pm$0.7 & 76.5$\pm$0.5 & 75.9$\pm$0.5 & 81.2$\pm$0.9 & 68.5$\pm$0.0 & 70.8$\pm$1.7 & 71.9$\pm$0.3 \\
\cellcolor{lightgray}FedProx       & 83.2$\pm$0.9 & 85.4$\pm$0.8 & 86.3$\pm$1.2 & 76.3$\pm$1.0 & 76.0$\pm$2.1 & 82.1$\pm$1.5 & 68.7$\pm$0.5 & 71.2$\pm$1.7 & 72.9$\pm$0.5 \\
\cellcolor{lightgray}Scaffold      & 83.7$\pm$1.0 & 84.6$\pm$0.8 & \textcolor{blue}{\underline{\textbf{86.7$\pm$1.1}}} & 75.3$\pm$0.4 & 76.9$\pm$1.1 & 82.1$\pm$0.7 & 68.3$\pm$0.4 & 70.6$\pm$1.0 & 71.3$\pm$1.1 \\
\cellcolor{lightgray}FedDisco      & 83.1$\pm$0.8 & \textcolor{blue}{\textbf{85.8$\pm$0.8}} & 86.5$\pm$1.4 & 78.8$\pm$1.1 & 81.7$\pm$0.5 & 80.9$\pm$0.9 & 69.4$\pm$1.4 & 71.4$\pm$0.7 & 72.3$\pm$0.9 \\
\cellcolor{lightgray}SFLV1         & 49.3$\pm$3.2 & 82.3$\pm$3.0 & 81.2$\pm$1.3 & 72.2$\pm$2.8 & 73.3$\pm$1.4 & 74.4$\pm$0.7 & 66.2$\pm$1.5 & 67.3$\pm$1.0 & 68.7$\pm$0.6 \\
\cellcolor{lightgray}SFLV2         & 81.7$\pm$1.5 & 81.6$\pm$2.9 & 81.0$\pm$1.4 & 73.8$\pm$0.6 & 74.0$\pm$1.1 & 75.1$\pm$1.2 & 67.9$\pm$1.1 & 67.6$\pm$1.2 & 68.5$\pm$1.0 \\
\cellcolor{lightgray}\textbf{TriCon-SF} & \textcolor{blue}{\textbf{85.5$\pm$1.1}} & 85.3$\pm$0.8 & 86.2$\pm$0.8 & \textcolor{blue}{\textbf{83.5$\pm$1.4}} & \textcolor{blue}{\textbf{82.8$\pm$2.2}} & \textcolor{blue}{\textbf{84.7$\pm$1.0}} & \textcolor{blue}{\textbf{80.9$\pm$1.1}} & \textcolor{blue}{\textbf{83.5$\pm$1.9}} & \textcolor{blue}{\textbf{84.5$\pm$1.0}} \\
\midrule

\multicolumn{10}{c}{\cellcolor{lightgray}\textbf{\textit{HAM10000}} - Centralized training ACC $=79.2\pm0.8$} \\
\midrule
\cellcolor{lightgray}FedAvg        & 75.0$\pm$1.0 & 74.5$\pm$1.0 & 76.0$\pm$0.5 & 72.5$\pm$0.3 & 73.3$\pm$0.3 & 76.3$\pm$0.5 & 71.4$\pm$0.5 & 71.1$\pm$0.4 & 76.9$\pm$0.6 \\
\cellcolor{lightgray}FedProx       & 76.6$\pm$0.7 & 77.4$\pm$0.5 & 77.9$\pm$1.4 & 74.4$\pm$0.2 & 74.3$\pm$0.6 & 78.1$\pm$0.4 & 70.1$\pm$0.3 & 74.4$\pm$0.1 & 77.5$\pm$0.6 \\
\cellcolor{lightgray}Scaffold      & 75.6$\pm$0.3 & 77.1$\pm$0.4 & 78.1$\pm$0.4 & 73.6$\pm$0.6 & 74.0$\pm$0.3 & \textcolor{blue}{\textbf{78.4$\pm$0.1}} & 69.8$\pm$0.2 & 74.0$\pm$0.5 & 76.6$\pm$0.5 \\
\cellcolor{lightgray}FedDisco      & \textcolor{blue}{\textbf{77.1$\pm$0.6}} & \textcolor{blue}{\textbf{77.6$\pm$0.9}} & \textcolor{blue}{\textbf{78.6$\pm$0.2}} & 75.3$\pm$0.2 & 76.1$\pm$0.6 & 77.9$\pm$0.3 & 69.9$\pm$0.2 & 75.2$\pm$0.5 & 77.1$\pm$0.4 \\
\cellcolor{lightgray}SFLV1         & 66.5$\pm$1.7 & 65.1$\pm$1.9 & 75.8$\pm$0.8 & 61.6$\pm$1.6 & 66.3$\pm$1.2 & 75.0$\pm$0.6 & 63.6$\pm$0.7 & 66.4$\pm$0.5 & 74.9$\pm$0.5 \\
\cellcolor{lightgray}SFLV2         & 64.8$\pm$3.2 & 75.0$\pm$1.6 & 77.6$\pm$0.8 & 69.4$\pm$1.5 & 68.4$\pm$1.6 & 75.8$\pm$0.8 & 69.0$\pm$1.5 & 74.1$\pm$0.8 & \textcolor{blue}{\underline{\textbf{79.7$\pm$0.8}}} \\
\cellcolor{lightgray}\textbf{TriCon-SF} & 76.3$\pm$0.3 & 76.2$\pm$0.6 & 78.1$\pm$1.1 & \textcolor{blue}{\textbf{76.5$\pm$0.7}} & \textcolor{blue}{\textbf{77.9$\pm$0.6}} & 77.5$\pm$0.6 & \textcolor{blue}{\textbf{75.4$\pm$0.4}} & \textcolor{blue}{\textbf{76.6$\pm$0.4}} & 78.1$\pm$0.4 \\
\midrule
\addlinespace[2pt]
\multicolumn{10}{p{\dimexpr0.9\textwidth-2\tabcolsep\relax}}{\textbf{Note:} Higher accuracy is better. The best accuracy in each column is marked in \textbf{\textcolor{blue}{blue}}. The overall best result for each dataset is \underline{underlined}.} \\
\bottomrule
\end{tabularx}
\end{table*}

\begin{figure*}[!htb]
    \centering
    \includegraphics[width=\textwidth]{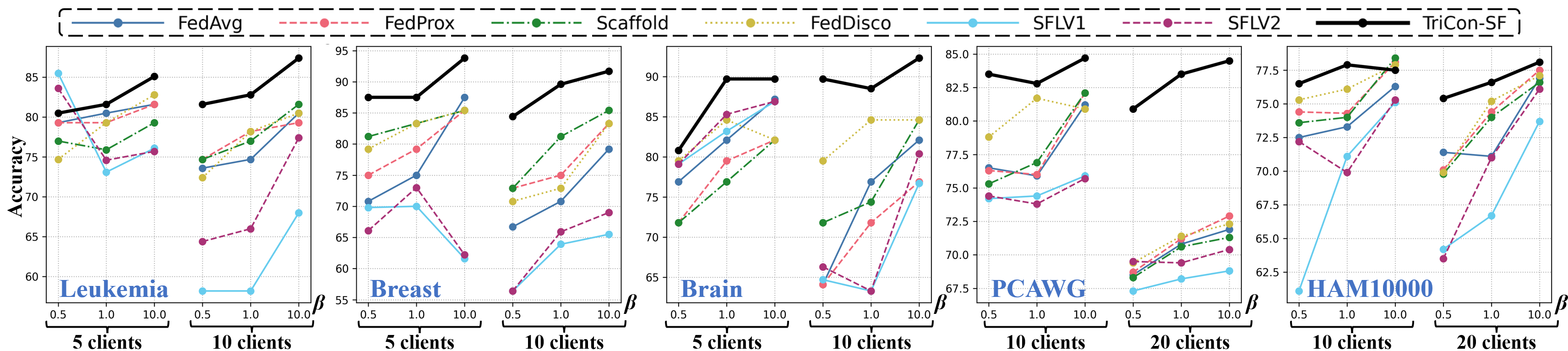}
    \caption{Test accuracy (\%) of different FL algorithms across five benchmark datasets.}
    \label{fig:acc_results}
\end{figure*}

\subsection{Accuracy evaluation}
We begin by evaluating the test accuracy of models trained using baseline algorithms across different client settings and levels of data heterogeneity. As shown in Tables \ref{tab:test_acc_cumida_final_transposed1} and \ref{tab:test_acc_cumida_final_transposed2}, as well as in the plots of Figure \ref{fig:acc_results}, TriCon-SF consistently achieves the highest accuracy across most datasets and experimental settings, demonstrating strong robustness and generalization.

Consistent with prior findings, we observe that performance typically degrades with increasing data heterogeneity (i.e., smaller $\beta$ values) and client numbers. This effect is particularly evident on more complex datasets such as PCAWG and HAM10000, where increasing the number of clients from 5 to 20 leads to substantial performance drops for most baseline methods. Notably, TriCon-SF remains resilient under such challenging settings, achieving over 83.5\% accuracy on PCAWG and 78.1\% on HAM10000 even at the largest scale (20 clients, $\beta$ = 10), while most baselines fall below 75\%.

Moreover, TriCon-SF consistently approximates centralized learning performance across all CUMIDA datasets. On Leukemia, for instance, TriCon-SF achieves 87.4\% accuracy at 10 clients and when $\beta$ = 10, only 1.1\% lower than the centralized result (88.5\%). Similarly, for the Brain dataset, TriCon-SF matches or exceeds 90\% accuracy under low-client settings (2 or 5 clients), with negligible loss compared to centralized training. Other methods, by contrast, typically suffer from 5\%–15\% accuracy drops under similar settings.

In summary, TriCon-SF not only exhibits strong performance in standard federated learning setups but also shows high robustness to distributional shifts and client scaling, making it a reliable solution for real-world distributed, multi-modal optimization scenarios.

\subsection{Communication cost evaluation}

We analyze the communication efficiency of the tested FL algorithms using three key metrics: communication rounds to convergence (R\#), total communication cost (C\#), and their normalized improvement rates over the baseline (R↑, C↑). To provide a clear basis for comparison, we define the total communication cost (C\#) based on the learning paradigm. Let $R_{\#}$ be the number of rounds to convergence, $n$ be the number of clients, and $M_{\text{size}}$ be the size of the model parameters.

For \textbf{parallel FL algorithms} (e.g., FedAvg, FedProx), where each client communicates with a central server in every round (one download, one upload), the total cost is calculated as:
\begin{equation}
    C_{\text{parallel}}^{\#} = R_{\#} \times n \times 2 \times M_{\text{size}}.
    \label{eq:cost_parallel}
\end{equation}

In contrast, for our \textbf{serial framework TriCon-SF}, a round consists of $n$ sequential client-to-client transmissions. The total cost is therefore:
\begin{equation}
    C_{\text{serial}}^{\#} = R_{\#} \times n \times M_{\text{size}}.
    \label{eq:cost_serial}
\end{equation}
This distinction highlights the inherent communication efficiency of the serial architecture. As shown in Table \ref{tab:fl_comparison_merged} and Figure \ref{fig:dp1} and Figure \ref{fig:dp2}, TriCon-SF significantly outperforms all baseline methods in both communication efficiency and cost-effectiveness across varying client counts and heterogeneity levels.

Across all $\beta$ settings and datasets, TriCon-SF consistently achieves the lowest R\# and C\#, indicating that it reaches target accuracy with far fewer communication rounds and lower overall cost. For example, under $\beta$ = 0.5 with 10 clients on PCAWG, TriCon-SF converges in just 18 rounds, yielding a 4.33× improvement in R↑ and a 2.17× improvement in C↑ over the FedAvg baseline. This trend persists even under more challenging conditions; at $\beta$ = 10 with 20 clients, TriCon-SF requires only 5 rounds on PCAWG and 13 rounds on HAM10000, reflecting dramatic improvements in both R↑ (15.4×, 12.85×) and C↑ (7.7×, 6.42×) respectively.

A particularly notable advantage of TriCon-SF is its scalability. While the communication rounds and costs of most baselines increase sharply with more clients, TriCon-SF maintains minimal growth. For example, when scaling from 10 to 20 clients under $\beta$ = 1 on PCAWG, TriCon-SF's round count decreases from 15 to 7 rounds, still outperforming all others, with a 11.29× improvement in R↑ and 5.25× in C↑ at 20 clients.
\begin{figure}[!htb]
    \centering
    \includegraphics[width=0.45\textwidth]{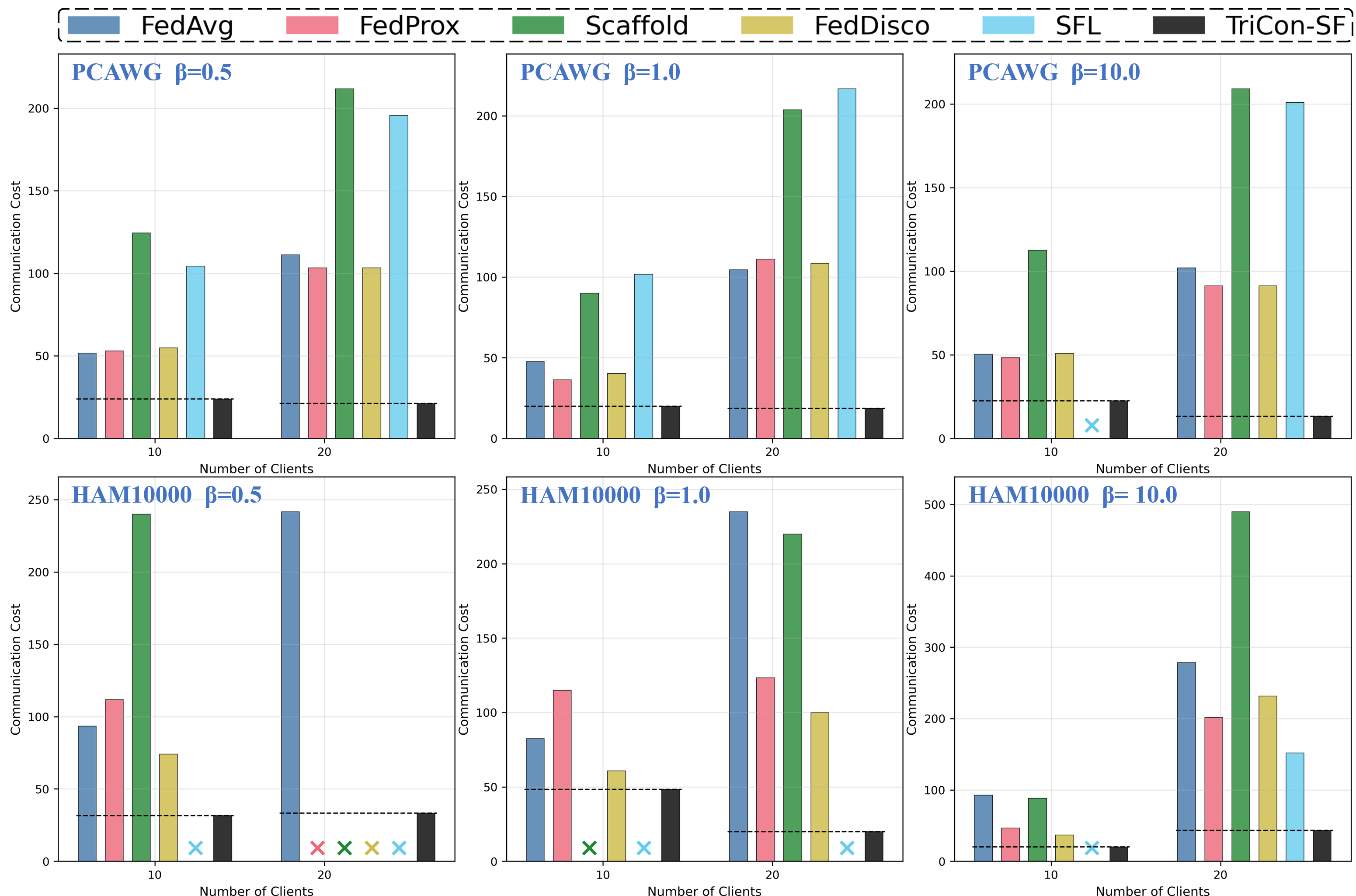}
    \caption{Comparison of communication cost needed for convergence across five benchmark datasets using different FL algorithms. Lower cost indicates higher communication efficiency.}
    \label{fig:dp1}
\end{figure}

\begin{figure}[!htb]
    \centering
    \includegraphics[width=0.45\textwidth]{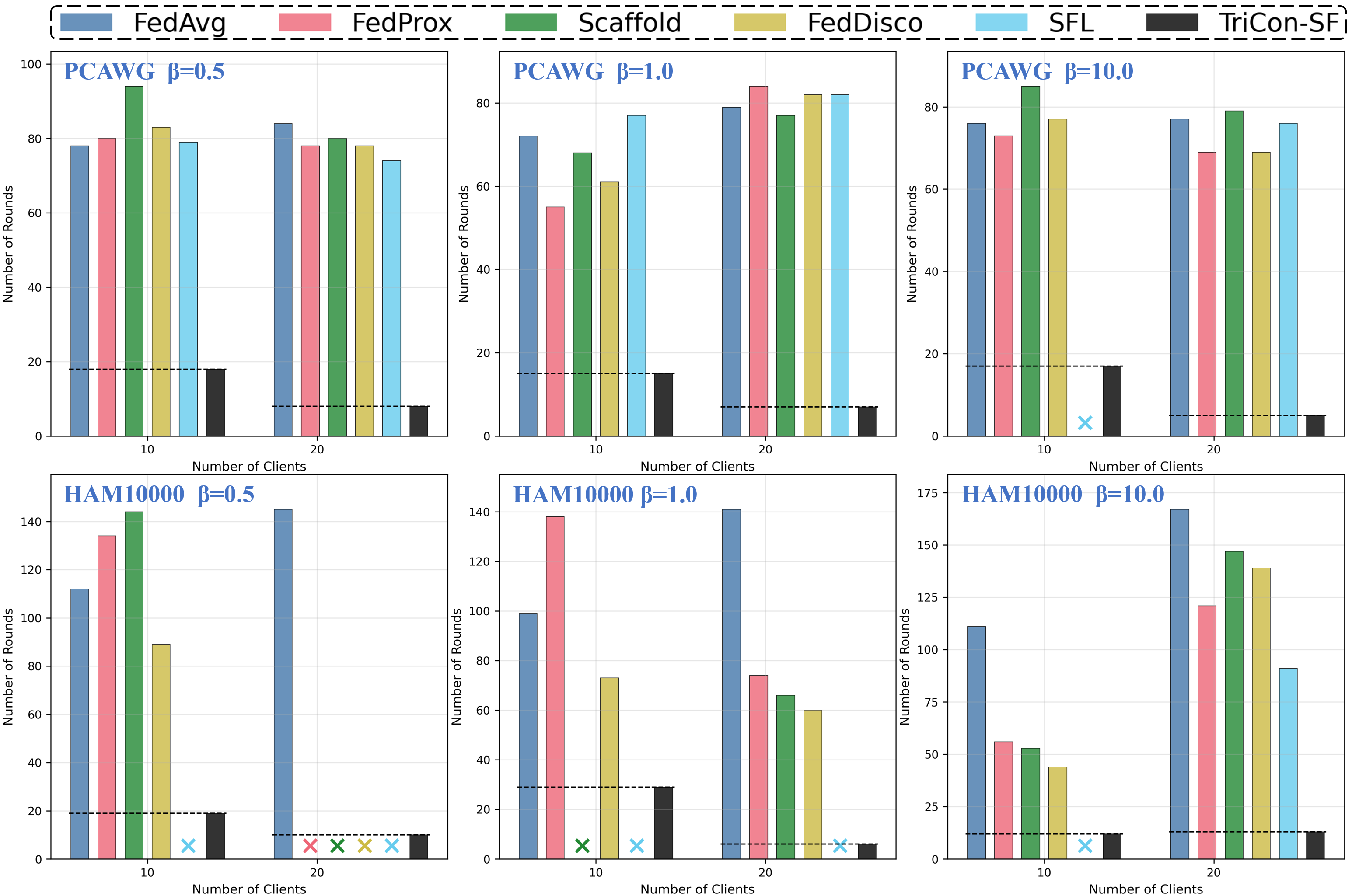}
    \caption{Comparison of communication rounds needed for convergence across five benchmark datasets using different FL algorithms. Lower round counts indicate faster convergence and higher communication efficiency.}
    \label{fig:dp2}
\end{figure}

\begin{figure}[!htb]
    \centering
    \includegraphics[width=0.45\textwidth]{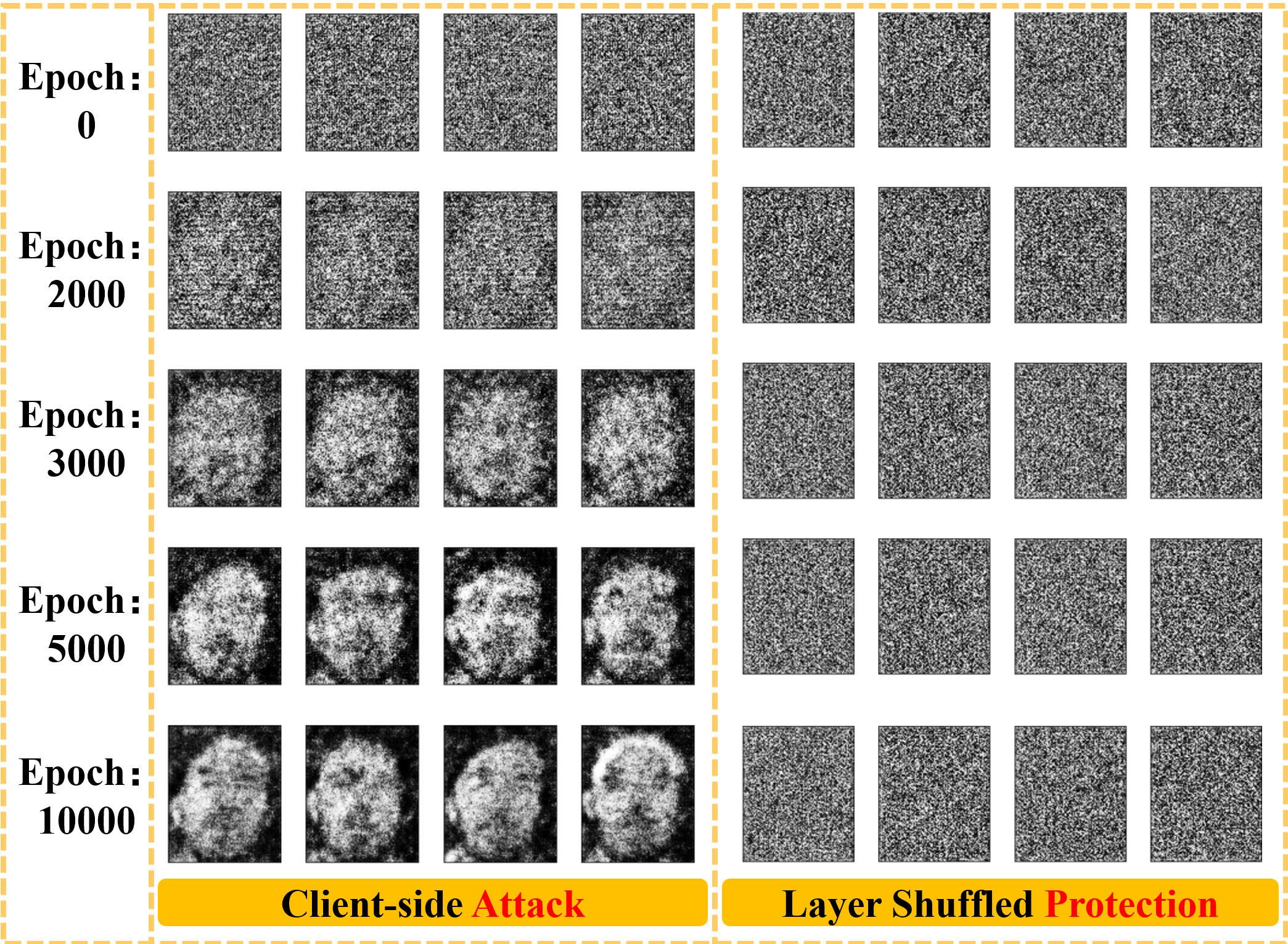}
    \caption{Visualization of client-side gradient inversion attacks on the ORL Face Dataset. Left: In a naive FL setup without protection, reconstructed images progressively reveal identifiable facial features as training epochs increase, demonstrating severe privacy leakage. Right: With our proposed layer shuffle mechanism combined with $L_2$ norm clipping and Gaussian noise injection, reconstructed images remain visually indistinguishable from noise across all epochs, effectively preventing identity inference.}
    \label{fig:dp_attack}
\end{figure}

\begin{table*}[!t]
\centering
\footnotesize
\setlength{\tabcolsep}{2.5pt}
\renewcommand{\arraystretch}{0.9}
\caption{Communication rounds (R\#), total costs (C\#), and improvement rates (R↑/C↑) on PCAWG and HAM10000.}
\label{tab:fl_comparison_merged}
\begin{tabularx}{\textwidth}{
    >{\raggedright\arraybackslash}p{2.2cm}|
    *{4}{>{\centering\arraybackslash}X}|
    *{4}{>{\centering\arraybackslash}X}|
    *{4}{>{\centering\arraybackslash}X}|
    *{4}{>{\centering\arraybackslash}X}@{}}
\toprule
\multirow{2}{*}{\textbf{ }} & \multicolumn{8}{c|}{\textbf{\textit{PCAWG}}} & \multicolumn{8}{c}{\textbf{\textit{HAM10000}}} \\
\cmidrule(lr){2-9} \cmidrule(lr){10-17}
& \multicolumn{4}{c|}{\textbf{10 Clients}} & \multicolumn{4}{c|}{\textbf{20 Clients}} & \multicolumn{4}{c|}{\textbf{10 Clients}} & \multicolumn{4}{c}{\textbf{20 Clients}} \\
\cmidrule(lr){2-5} \cmidrule(lr){6-9} \cmidrule(lr){10-13} \cmidrule(lr){15-17}
& R\# & R↑ & C\# & C↑ & R\# & R↑ & C\# & C↑ & R\# & R↑ & C\# & C↑ & R\# & R↑ & C\# & C↑ \\
\midrule

\rowcolor{lightgray}
\textbf{$\beta=0.5$} & \multicolumn{4}{c|}{\textit{Target ACC: 72.6\%}} & \multicolumn{4}{c|}{\textit{63.34\%}} & \multicolumn{4}{c|}{\textit{72.59\%}} & \multicolumn{4}{c}{\textit{72.16\%}} \\
\midrule
\cellcolor{lightgray}FedAvg          & 78 & 1.00× & 51.592 & 1.00× & 84 & 1.00× & 111.121 & 1.00× & 112 & 1.00× & 93.294 & 1.00× & 145 & 1.00× & 241.565 & 1.00× \\
\cellcolor{lightgray}FedProx         & 80 & 0.97× & 52.915 & 0.97× & 78 & 1.08× & 103.183 & 1.08× & 134 & 0.84× & 111.62 & 0.84× & N/A & N/A & N/A & N/A \\
\cellcolor{lightgray}Scaffold        & 94 & 0.83× & 124.349 & 0.41× & 80 & 1.05× & 211.658 & 0.53× & 144 & 0.78× & 239.899 & 0.39× & N/A & N/A & N/A & N/A \\
\cellcolor{lightgray}FedDisco        & 83 & 0.94× & 54.899 & 0.94× & 78 & 1.08× & 103.183 & 1.08× & 89 & 1.26× & 74.135 & 1.26× & N/A & N/A & N/A & N/A \\
\cellcolor{lightgray}SFL             & 79 & 0.99× & 104.29 & 0.49× & 74 & 1.14× & 195.51 & 0.57× & N/A & N/A & N/A & N/A & N/A & N/A & N/A & N/A \\
\cellcolor{lightgray}\textbf{TriCon-SF} & \textcolor{blue}{\textbf{18}} & \textcolor{blue}{\textbf{4.33×}} & \textcolor{blue}{\textbf{23.812}} & \textcolor{blue}{\textbf{2.17×}} & \textcolor{blue}{\textbf{8}} & \textcolor{blue}{\textbf{10.5×}} & \textcolor{blue}{\textbf{21.166}} & \textcolor{blue}{\textbf{5.25×}} & \textcolor{blue}{\textbf{19}} & \textcolor{blue}{\textbf{5.89×}} & \textcolor{blue}{\textbf{31.653}} & \textcolor{blue}{\textbf{2.95×}} & \textcolor{blue}{\textbf{10}} & \textcolor{blue}{\textbf{14.5×}} & \textcolor{blue}{\textbf{33.319}} & \textcolor{blue}{\textbf{7.25×}} \\
\midrule

\rowcolor{lightgray}
\textbf{$\beta=1$} & \multicolumn{4}{c|}{\textit{Target ACC: 70.33\%}} & \multicolumn{4}{c|}{\textit{65.44\%}} & \multicolumn{4}{c|}{\textit{73.73\%}} & \multicolumn{4}{c}{\textit{71.65\%}} \\
\midrule
\cellcolor{lightgray}FedAvg          & 72 & 1.00× & 47.623 & 1.00× & 79 & 1.00× & 104.506 & 1.00× & 99 & 1.00× & 82.465 & 1.00× & 141 & 1.00× & 234.901 & 1.00× \\
\cellcolor{lightgray}FedProx         & 55 & 1.31× & 36.379 & 1.31× & 84 & 0.94× & 111.121 & 0.94× & 138 & 0.72× & 114.952 & 0.72× & 74 & 1.91× & 123.281 & 1.91× \\
\cellcolor{lightgray}Scaffold        & 68 & 1.06× & 89.955 & 0.53× & 77 & 1.03× & 203.721 & 0.51× & N/A & N/A & N/A & N/A & 66 & 2.14× & 219.907 & 1.07× \\
\cellcolor{lightgray}FedDisco        & 61 & 1.18× & 40.347 & 1.18× & 82 & 0.96× & 108.475 & 0.96× & 73 & 1.36× & 60.808 & 1.36× & 60 & 2.35× & 99.958 & 2.35× \\
\cellcolor{lightgray}SFL             & 77 & 0.94× & 101.72 & 0.47× & 82 & 0.96× & 216.65 & 0.48× & N/A & N/A & N/A & N/A & N/A & N/A & N/A & N/A \\
\cellcolor{lightgray}\textbf{TriCon-SF} & \textcolor{blue}{\textbf{15}} & \textcolor{blue}{\textbf{4.8×}} & \textcolor{blue}{\textbf{19.843}} & \textcolor{blue}{\textbf{2.4×}} & \textcolor{blue}{\textbf{7}} & \textcolor{blue}{\textbf{11.29×}} & \textcolor{blue}{\textbf{18.52}} & \textcolor{blue}{\textbf{5.64×}} & \textcolor{blue}{\textbf{29}} & \textcolor{blue}{\textbf{3.41×}} & \textcolor{blue}{\textbf{48.313}} & \textcolor{blue}{\textbf{1.71×}} & \textcolor{blue}{\textbf{6}} & \textcolor{blue}{\textbf{23.5×}} & \textcolor{blue}{\textbf{19.992}} & \textcolor{blue}{\textbf{11.75×}} \\
\midrule

\rowcolor{lightgray}
\textbf{$\beta=10$} & \multicolumn{4}{c|}{\textit{Target ACC: 77.88\%}} & \multicolumn{4}{c|}{\textit{65.77\%}} & \multicolumn{4}{c|}{\textit{75.57\%}} & \multicolumn{4}{c}{\textit{75.87\%}} \\
\midrule
\cellcolor{lightgray}FedAvg          & 76 & 1.00× & 50.269 & 1.00× & 77 & 1.00× & 101.861 & 1.00× & 111 & 1.00× & 92.461 & 1.00× & 167 & 1.00× & 278.216 & 1.00× \\
\cellcolor{lightgray}FedProx         & 73 & 1.04× & 48.285 & 1.04× & 69 & 1.12× & 91.278 & 1.12× & 56 & 1.98× & 46.647 & 1.98× & 121 & 1.38× & 201.582 & 1.38× \\
\cellcolor{lightgray}Scaffold        & 85 & 0.89× & 112.443 & 0.45× & 79 & 0.97× & 209.013 & 0.49× & 53 & 2.09× & 88.296 & 1.05× & 147 & 1.14× & 489.794 & 0.57× \\
\cellcolor{lightgray}FedDisco        & 77 & 0.99× & 50.93 & 0.99× & 69 & 1.12× & 91.278 & 1.12× & 44 & 2.52× & 36.651 & 2.52× & 139 & 1.2× & 231.569 & 1.2× \\
\cellcolor{lightgray}SFL             & N/A & N/A & N/A & N/A & 76 & 1.01× & 200.80 & 0.51× & N/A & N/A & N/A & N/A & 91 & 1.84× & 151.79 & 1.83× \\
\cellcolor{lightgray}\textbf{TriCon-SF} & \textcolor{blue}{\textbf{17}} & \textcolor{blue}{\textbf{4.47×}} & \textcolor{blue}{\textbf{22.489}} & \textcolor{blue}{\textbf{2.24×}} & \textcolor{blue}{\textbf{5}} & \textcolor{blue}{\textbf{15.4×}} & \textcolor{blue}{\textbf{13.229}} & \textcolor{blue}{\textbf{7.7×}} & \textcolor{blue}{\textbf{12}} & \textcolor{blue}{\textbf{9.25×}} & \textcolor{blue}{\textbf{19.992}} & \textcolor{blue}{\textbf{4.62×}} & \textcolor{blue}{\textbf{13}} & \textcolor{blue}{\textbf{12.85×}} & \textcolor{blue}{\textbf{43.315}} & \textcolor{blue}{\textbf{6.42×}} \\
\bottomrule
\addlinespace[2pt]
\multicolumn{17}{p{\dimexpr \textwidth-2\tabcolsep\relax}}{\footnotesize\textbf{Note:} Lower R\# and C\# are better; higher R↑ and C↑ are better. Best results are \textbf{\textcolor{blue}{blue}}. "N/A" means the method failed to reach the target convergence accuracy. "1.00×" is baseline (FedAvg).}
\end{tabularx}
\end{table*}

\subsection{Security evaluation}
To assess the privacy-preserving effectiveness of our layer shuffle mechanism in TriCon-SF, we perform visual inversion attacks on the ORL Face Dataset \cite{faruqe2009face}, a standard benchmark comprising 400 grayscale images (92×112 pixels) from 40 individuals. The dataset introduces realistic variations in facial expressions, poses, illumination, and occlusion, making it well-suited for evaluating model vulnerability under FL.

We adopt a client-side gradient inversion attack, targeting intermediate features shared during FL training to reconstruct private data. As shown in the left panel of Figure~\ref{fig:dp_attack}, the naive FL setup (without protection) permits progressive identity leakage: coarse facial structures emerge as early as Epoch 3000, and detailed reconstructions become visually identifiable by Epochs 5000 and 10000. This highlights the significant privacy risks in high-dimensional FL scenarios.

To counteract this, we incorporate a layer shuffle mechanism that disrupts representational structure by randomly permuting hidden layer outputs. This is coupled with $L_2$ norm clipping (threshold $c = 1.0$) and Gaussian noise injection (zero-mean, $\sigma = 1.0$) to enforce differential privacy. Clipping ensures bounded sensitivity by limiting each feature vector's magnitude, while noise injection adds uncertainty to obfuscate reconstruction.

The right panel of Figure~\ref{fig:dp_attack} confirms the effectiveness of our approach: even after 10,000 epochs, reconstructed images remain visually indistinguishable from noise, with no identifiable patterns. These results validate that layer shuffling introduces structural obfuscation that blocks attack viability and offers strong empirical resilience against client-side privacy attacks in FL.

\section{CONCLUSION}
In this paper, we propose a novel paradigm for serial FL, termed TriCon-SF, which integrates a three-layer shuffle mechanism with a contribution-aware evaluation strategy. This design enhances both the robustness and accountability of the FL process. Experimental results across five datasets demonstrate that TriCon-SF consistently outperforms existing FL baselines in terms of accuracy and computational efficiency. Moreover, both theoretical proofs and empirical analyses confirm the framework’s resilience against gradient leakage attacks and its ability to detect malicious clients.

Despite these promising results, the primary limitation of TriCon-SF lies in the trade-off between handling data heterogeneity and mitigating catastrophic forgetting. As such, future work will focus on integrating TriCon-SF with continual learning techniques to address the challenge of catastrophic forgetting in serial FL settings.

\bibliographystyle{IEEEtran}
\bibliography{ref}

\end{document}